\pdfoutput=1

\documentclass[twoside]{article}
\usepackage[preprint]{aistats2027}
\usepackage{fancyhdr,amsmath,amssymb,graphicx,xcolor}

\usepackage[round,authoryear]{natbib}
\usepackage{hyperref}

\usepackage{framed}
\usepackage{capt-of}

\newif\ifshownotes
\shownotestrue %
\newcounter{notecounter}
\newcounter{subnotecounter}[notecounter]

\newif\ifinsidenote

\def\namealgone{\texttt{Alg1}}
\def\namealgtwo{\texttt{Alg2}}

\newcommand{\ie}{\textit{i.e.}}
\newcommand{\eg}{\textit{e.g.}}

\usepackage{etoolbox}

\makeatletter
\newcommand{\toc@entrylevel}{subsection}
\pretocmd{\section}{\gdef\toc@entrylevel{subsection}}{}{}
\pretocmd{\subsection}{\gdef\toc@entrylevel{subsubsection}}{}{}
\pretocmd{\subsubsection}{\gdef\toc@entrylevel{paragraph}}{}{}
\pretocmd{\paragraph}{\gdef\toc@entrylevel{subparagraph}}{}{}
\pretocmd{\subparagraph}{\gdef\toc@entrylevel{subparagraph}}{}{}

\newcommand{\tocparagraph}[1]{%
  \let\toc@savedentrylevel\toc@entrylevel
  \par\phantomsection
  \paragraph*{#1}%
  \addcontentsline{toc}{\toc@savedentrylevel}{\protect\numberline{}Paragraph: #1}%
  \global\let\toc@entrylevel\toc@savedentrylevel
}

\newif\ifunited@jmlr
\@ifclassloaded{jmlr}{\united@jmlrtrue}{\united@jmlrfalse}
\ifunited@jmlr
  \newtheorem{assumption}[theorem]{Assumption}%
  \newcommand{\tocthmline}[2]{%
    \phantomsection
    \addcontentsline{toc}{\toc@entrylevel}{%
      \protect\numberline{}#1\ifx\relax#2\relax\else: #2\fi}}%
  \apptocmd{\@begintheorem}{\tocthmline{#1}{}}{}%
    {\PackageError{united_macros}{could not patch \string\@begintheorem}{}}%
  \apptocmd{\@opargbegintheorem}{\tocthmline{#1}{#3}}{}%
    {\PackageError{united_macros}{could not patch \string\@opargbegintheorem}{}}%
\else
  \RequirePackage{amsthm}%
  \RequirePackage{thmtools}%
  \RequirePackage{thm-restate}%
  \newcommand{\tocthmline}{%
    \addcontentsline{toc}{\toc@entrylevel}{\protect\numberline{}%
      \thmt@thmname\@ifundefined{thmt@optarg}{}{\ifx\thmt@optarg\@empty\else: \thmt@optarg\fi}}}%
  \declaretheorem[style=plain,name=Theorem,postheadhook=\tocthmline,numberwithin=section]{theorem}%
  \declaretheorem[style=plain,name=Proposition,postheadhook=\tocthmline,sibling=theorem]{proposition}%
  \declaretheorem[style=plain,name=Remark,postheadhook=\tocthmline,numberwithin=section]{remark}%
\fi
\makeatother

\newcommand{\x}{\times}

\newcommand{\MDP}{\mathcal{M}}
\newcommand{\statespace}{\mathcal{S}}
\newcommand{\state}{s}
\newcommand{\State}{S}
\newcommand{\cemeterystate}{\state_\dagger}
\newcommand{\actionspace}{\mathcal{A}}
\newcommand{\action}{a}
\newcommand{\Action}{A}
\newcommand{\transition}{P}
\newcommand{\reward}{r}
\newcommand{\randreward}{R}
\newcommand{\horizon}{H}
\newcommand{\mask}{m}
\newcommand{\maskpar}{u}

\newcommand{\classifier}{h}
\newcommand{\hypclass}{\Hc}
\newcommand{\feature}{\phi}
\newcommand{\Feature}{\Phi}
\newcommand{\transitionfeature}{\psi}
\newcommand{\rewardfeature}{\omega}
\newcommand{\cemetery}{s_\dagger}
\newcommand{\featuremask}{\widetilde{\feature}}
\newcommand{\featurespace}{\mathcal{W}}
\newcommand{\maskclass}{\mathcal{G}}
\newcommand{\policy}{\pi}
\newcommand{\regret}{\mathscr{R}}
\newcommand{\episode}{k}
\newcommand{\episodes}{K}
\newcommand{\projection}{\Pi}
\newcommand{\gap}{\Delta}
\newcommand{\diameter}{D}

\newcommand{\data}{\Dc}

\newcommand{\startlaw}{\varrho}
\newcommand{\pair}{z}
\newcommand{\pairspace}{Z}
\newcommand{\pairfeature}[1]{\feature_{#1}}
\newcommand{\maskMDP}[1]{\MDP_{#1}}
\newcommand{\validactions}[1]{\actionspace_{#1}}
\newcommand{\maskedreward}[1]{\widetilde{\reward}_{#1}}
\newcommand{\maskedtransition}[1]{\widetilde{\transition}_{#1}}
\newcommand{\firstinvalid}{\tau}
\newcommand{\revealset}{\mathrm{Rev}}
\newcommand{\episodecost}{c}
\newcommand{\versionstate}{\mathcal{V}}
\newcommand{\observedpairs}{\data}
\newcommand{\observedspan}{U}
\newcommand{\decideddirections}{U^{+}}
\newcommand{\spn}{\operatorname{span}}
\newcommand{\supportlower}{\sigma^{-}}
\newcommand{\supportupper}{\sigma^{+}}
\newcommand{\certvalid}{\mathcal{D}^{+}}
\newcommand{\certinvalid}{\mathcal{D}^{-}}
\newcommand{\undecided}{\mathscr{U}}
\newcommand{\undecidedrank}[1]{\dim\featurematrix(\undecided(#1))}
\newcommand{\markovpolicies}{\Pi^{\mathrm M}}
\newcommand{\bestversion}{\mathcal{U}}
\newcommand{\minreachprob}{p_{\min}}
\newcommand{\undecidedhittingtime}{\tau_{\undecided}}
\newcommand{\episodicfiltration}{\mathcal{F}}
\newcommand{\withinhistory}{\mathcal{H}}
\newcommand{\signpattern}{\sigma}
\newcommand{\cells}{\operatorname{cells}}
\newcommand{\optimisticgate}{\widehat{\actionspace}}
\newcommand{\optimisticvalue}{\widehat V}
\newcommand{\optimisticq}{\widehat Q}
\newcommand{\absorptionevent}{E}
\newcommand{\probingmistake}{M}
\newcommand{\cumulativeregret}{R}

\newcommand{\route}{\rho}
\newcommand{\routes}{\Pi}
\newcommand{\abortaction}{a^{\circ}}
\newcommand{\undecidedspan}{\mathsf{S}}
\newcommand{\undecidedpotential}[1]{\Psi_{#1}}
\newcommand{\potential}{\Psi}
\newcommand{\efficiencycoef}{\Gamma}
\newcommand{\price}{\mathfrak{P}}
\newcommand{\priceoperator}{\mathrm{B}}
\newcommand{\stoppingepisode}{\varsigma}
\newcommand{\certificate}{\widehat{\bestversion}}
\newcommand{\pessimisticvalue}{V^{-}}
\newcommand{\certifiedactions}{\actionspace^{-}}
\newcommand{\certifiedmodel}{\MDP^{-}}
\newcommand{\certifiedpolicy}{\policy^{-}}
\newcommand{\collected}{\operatorname{col}}
\newcommand{\bonusvalue}{J}
\newcommand{\continuationvalue}{K}
\newcommand{\certprice}{\lambda^{-}}
\newcommand{\optprice}{\widehat{\lambda}}
\newcommand{\optshortfall}{\widehat{c}}
\newcommand{\target}{\mathscr{T}}
\newcommand{\aliasedclass}{\mathcal{T}}
\newcommand{\informative}{\mathcal{I}}
\newcommand{\staticbound}{\Lambda}
\newcommand{\cellvalue}{\mathring{V}}
\newcommand{\realprice}{\mathring{\lambda}}
\newcommand{\realcertificate}{\mathring{\bestversion}}
\newcommand{\realshortfall}{\mathring{c}}
\newcommand{\cellpolicy}{\policy^{\gamma}}
\newcommand{\refutationprob}{\mathsf{p}}
\newcommand{\valueloss}{c}
\newcommand{\rivals}{O}
\newcommand{\partialdirections}{l}

\newcommand{\falscomplexity}{\mathsf{C}_{\mathrm{fals}}}
\newcommand{\harmonic}[1]{\mathsf{H}_{#1}}

\newcommand{\inner}[2]{\langle #1,#2\rangle}

\usepackage{mathrsfs}
\usepackage[capitalize,noabbrev]{cleveref}
\crefname{assumption}{Assumption}{Assumptions}
\Crefname{assumption}{Assumption}{Assumptions}

\providecommand{\abs}[1]{\left\lvert #1 \right\rvert}

\providecommand{\de}{\mathrm{d}}
\providecommand{\R}{\mathbb{R}}
\providecommand{\Pb}{\mathbb{P}}

\providecommand{\E}{\mathbb{E}}
\providecommand{\1}{\mathbf{1}}
\providecommand{\Hc}{\mathcal{H}}

\newcommand{\baseMDP}{\MDP_{\mathrm{source}}}
\newcommand{\taskMDP}[1]{\MDP_{#1}}
\newcommand{\masksequence}{\Theta}
\newcommand{\history}{\mathcal{H}}
\newcommand{\dataset}{\mathcal{D}}
\newcommand{\filtration}{\mathcal{F}}
\newcommand{\confset}{\mathcal{C}}
\newcommand{\compatset}{\mathcal{BC}}
\newcommand{\transferinput}{\mathcal{I}}
\newcommand{\algorithmclass}{\mathbb{A}}
\newcommand{\taskalgorithm}{\mathfrak{A}}
\newcommand{\metapolicy}{\boldsymbol{\Pi}}

\newcommand{\augstatespace}{\statespace^{+}}
\newcommand{\featurematrix}{\phi}
\newcommand{\statefeaturematrix}[1]{\featurematrix(#1)}
\newcommand{\transitionfeaturematrix}{\mathbf{\Psi}}
\newcommand{\maskmatrix}{\mathbf{V}}
\newcommand{\policymatrix}{\mathbf{\Pi}}
\newcommand{\survivalkernel}{T}

\DeclareMathOperator*{\argmax}{argmax}
\DeclareMathOperator*{\argmin}{argmin}

\DeclareMathOperator{\sgn}{sgn}

\newenvironment{proofof}[1]{\begin{proof}[Proof of {#1}]}{\end{proof}}

\input{macro_table.tex}

    \def\Eb{\mathbb{E}}

    \def\Nb{\mathbb{N}}
    
    \def\Pb{\mathbb{P}}
    
    \def\Rb{\mathbb{R}}

    \def\Ac{\mathcal{A}}
    \def\Bc{\mathcal{B}}
    \def\Cc{\mathcal{C}}
    \def\Dc{\mathcal{D}}

    \def\Hc{\mathcal{H}}
    
    \def\Ic{\mathcal{I}}

    \def\Oc{\mathcal{O}}

    \def\As{\mathscr{A}}

    \def\Ts{\mathscr{T}}

\usepackage{algorithm,algorithmic}

\begin{document}

\twocolumn[\aistatstitle{Confessions of a Masked MDP}

\aistatsauthor{ Sebastian Reboul \And Lorenzo Croissant \And Hélène Halconruy }

\aistatsaddress{ SAMOVAR, Télécom SudParis, \\ Institut Polytechnique de Paris \\ Wiremind Cargo
  \And AI \& ML Group \\ Department of Engineering \\ Universitat Pompeu Fabra
  \And SAMOVAR, Télécom SudParis, \\ Institut Polytechnique de Paris \\ Modal'X, Université Paris-Nanterre } ]

\begin{abstract}

In reinforcement learning, one presumes the action set to be known ahead of time, and, thus, one must know if any state-action pairs are not permissible ahead of computations. 
In some applications, determining this action \emph{mask} is the main computational bottleneck. 
Isolating this problem to highlight its core properties, we study how to combine the search for an optimal policy with the learning of the mask itself, under a linear MDP model in which each attempted action reveals the exact value of a linear score determining its validity. In this framework, we first construct an  algorithm based on deterministic optimistic optimization and show that it is minimax optimal, which we show is also computationally efficient. 
However, it is sub-optimal on non-degenerate instances, as we show.
We propose an alternative, which is also computationally efficient, which plans its  exploration strategy. 
This algorithm admits a formal interpretation as an analogue of information-directed-sampling for deterministic feedback, which may be of independent interest.
Finally, we study an application of this framework to transfer learning, where the invalid action set changes over time. 

\end{abstract}

\section{Introduction}

Reinforcement learning, the interactive learning of control systems and games \citep{tesauro1995temporal,sutton2018reinforcement}, and in particular its formalisation using Markov Decision Processes (MDPs) \citep{puterman1994markov}, has been an undeniably successful field of research in the last thirty years. Applications span far and wide, from datacenter cooling \citep{gao2014machine}, to autonomous driving \citep{shalev2016safe}, through advertising \citep{liu2017deep} and conversational AI training \citep{li2016deep}. 

In the MDP formalisation, the set of available actions is typically assumed to be independent of the state of the system. In the classical theory \citep{bellman1957dynamic}, this assumption is a convenience and varying actions can be accommodated at the cost of some burdensome notation \citep[see, \eg{},][Ch.~6]{puterman1994markov}.

In computation, however, ever since RL methods have shifted to GPU-accelerated computation \citep{mnih2015human}, varying actions have become nefarious spoilers. 
Indeed, GPU-acceleration of, \eg{}, Value Iteration, relies on the ability to broadcast operations over the entire state-action space.  
The solution, early adopted by practitioners, is to use a boolean \emph{invalid action mask} to zero-out the rewards (and/or transitions) corresponding to invalid actions. Hadamard multiplication by the mask preserves the tensor structure, and is accelerated. Remarkably, there has been some theoretical work showing that masking preserves the convergence of Policy Gradient \citep{huang2020closer}.

This state of affairs is acceptable in many applications. However, in some cases, computing the mask is itself the computational bottleneck. For instance, in civil engineering applications, each action corresponds to a refinement of the current design (state), and determining whether it satisfies the project’s design constraints may require extensive simulations. Precomputing all possible designs in order to check their feasibility for a building is, quite evidently, a prohibitively expensive approach to architecture.  

When pre-computation of the mask is impossible, the current approach is to penalise invalid actions through reward shaping \citet{Ng1999Policy}, which is a form of Lagrangian relaxation. However, when it comes to learning, this approach is fundamentally flawed: it treats the \emph{deterministic} signal of permissibility and the stochastic reward signal homogeneously, a dilution which explains the terrible performance of these as \emph{learning} methods.

One would expect the literature on \emph{constrained} MDPs to have addressed this problem somehow. Yet, both classical \citep{altman1999constrained} and more recent \citep{achiam2017constrained,amani_safe_2021} work focuses on expected costs, with the aim of controlling violations. In our setting, due to determinism of signals, there is no reason to repeat a known violation. The closest work to ours is, thus, that of \citet{wei_safe_2023}, whose instantaneous hard constraints are in the same spirit, but whose regret bounds are dominated by the minimax analysis and stochasticity of the rewards in the interior of the constraints, which prevents the identification of optimal exploration strategies.

At the end of the day, we are left to face an MDP with masked actions, a \emph{masked} MDP, if you will, with the problem of how to efficiently explore it to identify the mask. Motivated by the computational cost of the mask itself, we focus on $0$-noise feedback, analogous to membership-query and realizable-classification settings \citep{angluin1988queries,littlestone1988learning}. We will show, with \cref{thm:separation}, that this is sufficient to achieve optimal regret if one must also learn the MDP dynamics and/or rewards. While the mask confesses its values freely, the problem of reaching the information efficiently through the MDP dynamics remains. 

In the remainder of this section, we detail the setting explicitly, and formally introduce the problem of mask learning. Thereafter, in \cref{sec:mask learning}, we establish a minimax-optimal algorithm based on Optimism in the Face of Uncertainty (OFU), and argue that this is not a good benchmark for this problem. In \cref{sec: instance dep}, we thus study finer, instance-dependent bounds, and introduce an algorithm which improves the exploration of OFU on non-degenerate instances, and connect it to the information-directed sampling (IDS) algorithm of \citet{russo2014information}, also identifying an auxiliary MDP formulation, \textit{à la } \citet{hao2022idrl}. Finally, in \cref{sec:transfert}, we discuss an example application to transfer learning of masks in MDPs, motivated by an industrial application.

\subsection{Setting}

As a basis, let $\MDP := (\statespace, \actionspace, \transition, \reward, \horizon)$ be a finite-horizon MDP, with state space $\statespace$, action space $\actionspace$, transition kernel $\transition:\statespace \times \actionspace\times\statespace \to [0,1]$, with $\sum_{s' \in \statespace} \transition(\state,\action,\state') = 1$, reward  $\reward:\statespace \times \actionspace \to [0,1]$,  and horizon $\horizon\in\Nb$.

We will adopt the formalism of \emph{linear MDPs} (without loss of generality as it subsumes tabular MDPs). Thus, we assume that there exists a known feature map $\feature:\statespace \times \actionspace \to \Rb^d$, $d\in\Nb$, such that, for all $(\state,\action) \in \statespace \times \actionspace$,
\[
\reward(\state,\action) = \feature(\state,\action)^\top \theta_r
\]
for some parameter $\theta_r \in \Rb^d$. Moreover, we assume that the transition kernel is linear in the features, \ie\ there is parameter-valued function $\theta_P:\State\to\Rb^d$ over $\statespace$ such that
\[
\transition(\state,\action,\state') = \feature (\state,\action)^\top \theta_P(\state')
\]
for all $(\state,\action,\state') \in \statespace \times \actionspace \times \statespace$.

From $\MDP$, we can define a family of \emph{masked} MDPs, indexed by the masks $m\in 2^{\statespace\x\actionspace}$, viewed either as function from $\statespace\x\actionspace\to\{0,1\}$ or a $\{0,1\}$-valued matrix in $\Rb^{\statespace\x\actionspace}$. The masked MDP $\MDP_m := (\statespace, \actionspace, \transition, \reward, \horizon, m)$ is defined on the augmented state space $\statespace\cup\{\cemeterystate\}$, where $\cemeterystate$ denotes an absorbing cemetery state. Its reward function is given for $(\state,\action)\in\statespace\x\actionspace$ by
\[
r_m(\state,\action)=m(\state,\action)r(\state,\action),
\]
with $r_m(\cemeterystate):=0$. For $(\state,\action,\state')\in\statespace\x\actionspace\x\statespace$, the transition kernel of $\MDP_\mask$ is defined by
\[
P_m(\state,\action,\state')=m(\state,\action)P(\state,\action,\state'), 
\]
with the remaining probability mass assigned to the cemetery state:
\[
P_m(\state,\action,\cemeterystate)
:=
1-\sum_{\state'\in\statespace}
m(\state,\action)P(\state,\action,\state').
\]
Finally, since $\cemeterystate$ is absorbing, $P_m(\cemeterystate,\cdot,\cemeterystate)=1$.

Note that a masked linear MDP remains a linear MDP, up to increasing the dimensions by $1$. Indeed the new feature map $\widetilde{\feature}:\statespace\cup\{\cemeterystate\}\x\actionspace\to\Rb^{d+1}$ is defined as $\widetilde{\feature}(\state,\action)
 =\bigl(m(\state,\action)\feature(\state,\action),1-m(\state,\action)\bigr)\in\Rb^{d+1}.$
Thus, in order to ease exposition and maintain consistent algebraic tools, we will also represent the mask linearly using the embedding $\feature $, \ie{} we assume that there is $\maskpar_\star\in\Rb^d$ such that 
\[ m(\state,\action)=\sgn(\feature (\state,\action)^\top \maskpar_\star) \quad\mbox{for all } (\state,\action)\in\statespace\x\actionspace.\]
When clear from context, we drop subscripts of $m$.

In a (masked) MDP, given any $h\in[\horizon]$ and $\policy:[\horizon]\x\statespace\to\actionspace$, let $(\State_h^{t,\state,\pi})_{h=t}^\horizon$ be the Markov chain induced by $\policy\in\Pi_M$ (Markov policies) and $P$ starting from $S_t^{t,\state,\pi} = \state$. It is well known that there is a cumulated-reward-maximising policy of this form, and we can solve the control problem by limiting ourselves to maximising, over $\pi\in\Pi_M$, the value functions $J_\pi$ defined for $(t,\state)\in [\horizon]\times\State$ by
\begin{equation*}
J_\pi(t,\state):=\Eb\left[\sum_{h=t}^{\horizon}r\bigl(\State_h^{t,\state,\pi},\pi_h(\State_h^{t,\state,\pi})\bigr)\right].
\end{equation*}
Moreover, we will denote the value function \[V:(t,\state)\in[\horizon]\x\statespace\mapsto V(t,\state):=\sup_{\policy}J_\policy(t,\state),\] where the supremum is taken over all Markov time-dependent policies $\policy$.
When necessary to clarify a mask on the MDP, we will add superscripts to $J$, $V$, or $\State$ to indicate the mask, \eg{} $J^m$, $V^m$, and $\State^{m,\policy}$. When considering an initial distribution $\startlaw $ over states, we will use the shorthand $V(h,\startlaw):=\int V(h,\state)\startlaw(\de\state)$ and likewise for $J$.

\section{Minimax mask learning}\label{sec:mask learning}

On the masked MDP, we can consider three main learning problems: the mask, the dynamics, and the combination of both. In this section, we draw a general picture of the mask learning problem with broad strokes, focusing on minimax-optimality.

\subsection{Learning problems on masked MDPs}

Clearly, if the mask is known but the dynamics must be learnt, the problem is statistically indistinguishable from a classical RL problem \citep{jin2020provably}. By comparison, we will focus on the other two problems where the mask is unknown, due to the absence of its precomputation, and an agent observes only sequences of state-action-reward-mask quadruples. 
Meanwhile, the learning problem of the dynamics is naturally formulated as an episodic online RL problem \citep{jin2020provably}, wherein visiting a state-action pair $(\state,\action)$ reveals a noisy signal of the reward and a sampled transition from $P(\state,\action,\cdot)$. Notice that, due to the presence of a cemetery state, one must consider an episodic setting, rather than a non-episodic one.

Let $\maskpar^\star$ denote the unknown true parameter of the mask's linear form.  Letting $\Ac$ be a learning algorithm which produces a sequence of action trajectories $\{(\Action_{k,h})_{h=1}^\horizon\}_{k\in\Nb}$ by playing policies $(\pi_{k})_{k\in\Nb}$, we measure its performance by its \emph{cumulative pseudo-regret}, which, after $K\in\Nb$ episodes, is defined as
\[ \regret_K(\Ac) := \sum_{k=1}^K \sup_{\pi} J_{\policy}(1,\startlaw) - J_{\policy_k}(1,\startlaw).\]

We continue, in \cref{subsec:minimax regret}, by the construction of a minimax-optimal optimistic algorithm. Then, in \cref{subsec: separation}, we show that any such algorithm is a sufficient building block for learning transitions and rewards if necessary. This will complete the picture of these three learning problems.

\subsection{Minimax Regret bounds}\label{subsec:minimax regret}

Naturally, one would begin a study of learning masks in a masked MDP by studying minimax-optimal regret. However, as \cref{thm: lower bound for mask learning} will show, this paradigm is a poor benchmark for understanding the moving pieces of the problem. Indeed, due to the linear embedding, the number of mistakes that can be made by any reasonable\footnote{Understood to mean an algorithm which never takes any action it could have known was masked.} algorithm is determined by the feature map in a tight manner. 

In order to keep track of available information about the mask, we will use a \emph{version space} \citep[{as in, \eg{},}][]{angluin1988queries} which contains all system hypotheses consistent with observed values of the mask. Let $\data\subset\statespace\x\actionspace$ be a collection of state-action pairs $(\state_i,\action_i)$, $i\in[n]$, and define $y\in\Rb^n$ by $y_i := \langle\feature(s_i,a_i) \vert \maskpar^\star\rangle$. The version space compatible with $\observedpairs$ is 
\begin{equation*}
 \versionstate(\observedpairs) = \confset_0 \,\cap\, \bigl\{\maskpar \in \R^d : \Phi_{\observedpairs}\maskpar = y\}
\end{equation*}
wherein $\Phi_{\observedpairs}=(\phi(s_i,a_i)^\top)_{i=1}^n$, and $\confset_0\subset\Rb^d$ is an a priori feasible set, say the unit ball. More broadly, we call $\versionstate\subset \Rb^d$ a version space if there is a set $\observedpairs\subset\statespace\x\actionspace$ such that $\versionstate=\versionstate(\observedpairs)$.

Whenever the learner plays $(\state,\action)\in\statespace\x\actionspace$, it observes the exact score $y_i = \feature(\state,\action)^\top\maskpar^\star$, (whether the action is valid or not), and thus $m(\state,\action)$, as well as the reward and next state of $\MDP_m$ (the next state being $\cemeterystate$ if the action is invalid). This noiseless feedback model is an extension of the deterministic bandit problems \citep{munos2011optimistic}.

Given a version space $\versionstate$, the linear structure of the mask partitions the state-action space into three sets: the \emph{certified valid} pairs,
\begin{equation*}
\certvalid(\versionstate):=\big\{(s,a)\in\statespace\x\actionspace:\inf_{\maskpar\in\versionstate}\feature(s,a)^\top\maskpar> 0\big\},
\end{equation*}
the \emph{certified invalid} pairs,
\begin{equation*}
\certinvalid(\versionstate):=\big\{(s,a)\in\statespace\x\actionspace:\sup_{\maskpar\in\versionstate}\feature(s,a)^\top\maskpar<0\big\},
\end{equation*}
and the remaining \emph{uncertified} pairs,
\begin{equation*}
\undecided(\versionstate):=(\statespace\times\actionspace)\setminus\bigl(\certvalid(\versionstate)\cup\certinvalid(\versionstate)\bigr).
\end{equation*}
In other words, a pair $(s,a)$ is \emph{uncertified} whenever the range of $\theta\in\versionstate\mapsto\langle\feature(s,a),\theta\rangle$ includes zero.

\begin{figure}[ht]
    \centering
    \includegraphics[width=.66\linewidth]{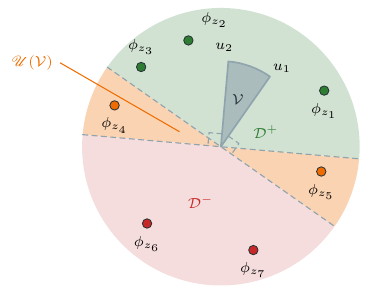}\label{fig:trichotomy}
    \captionof{figure}{Visualisation of $\versionstate$ and $(\certvalid,\certinvalid,\undecided)$.}
\end{figure}

As the learner plays, it generates a non-decreasing sequence of sets $(\data_t)_{t\ge1}$ by interacting with the environment, whose parameter $\maskpar^\star$ lies in the intersection of the non-increasing sequence $(\versionstate_t)_{t\ge1}$, where $\versionstate_t:=\versionstate(\data_t)$.

Given $\observedpairs$, $\versionstate(\observedpairs)$  is defined by linear equalities (up to intersection with $\confset_0$), so that sampling a new pair $(s,a)$ cannot affect the version space unless $\phi(s,a)$ is linearly independent of $\observedpairs$. 
This implies that the number of mistakes made by a reasonable algorithm (that never plays in $\certinvalid(\versionstate_t)$) is at most $$\dim\feature(\undecided(\versionstate_0))\le d\le\abs{\statespace}\abs{\actionspace}.$$
This quantity will be the central focus of our exploration analysis, however, it should be noted features linearly dependent on $\feature(\versionstate_t)$ can still identify the validity of other pairs (for instance $\phi_{z_4}$ and $\phi_{z_5}$ in \cref{fig:trichotomy}), which we will use to improve exploration. 

\begin{restatable}{theorem}{lowerb}
\label{thm: lower bound for mask learning}
    Let $\confset_0=\Bc_2(0,1)$ be the Euclidean ball. For every $d\ge 2$ and every $\horizon\ge 1$, there is a masked MDP with $\undecidedrank{\confset_0}=d$ on which any algorithm satisfies
    \[
        \sup_{\maskpar^\star\in\confset_0}
        \E_{\maskpar^\star}
        \bigl[\cumulativeregret(T)\bigr]
        \ge
        \frac{\horizon}{4}\min\{T,d-1\}.
    \]
\end{restatable}

\begin{proof}
    Consider the masked MDP in \cref{fig:doors}. We take the trivial feature map, so that $d=2\abs{\actionspace}$. Under the best possible mask, all actions in $s_1$ go to $s_2$ with reward $1$. Under the real mask, only one of the actions is valid in $s_1$. Any algorithm has no choice but to take actions at random in $s_1$ until $s_2$ is reached, which incurs regret $\horizon$ each time. The bound follows from expected number of trials, see Appendix \ref{proof:lowerb}.
\end{proof}

\begin{figure}[ht]
    \vspace{0pt}
    \centering
    \includegraphics[width=.5\linewidth]{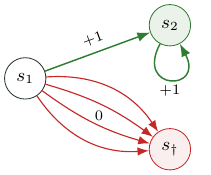}
    \captionof{figure}{The MDP in the proof of \cref{thm: lower bound for mask learning}.}
    \label{fig:doors}
\end{figure}

\noindent
In an unmasked MDP, optimistic algorithms offer a simple-to-analyse class of minimax-optimal algorithms \citep{auer2006logarithmic,jaksch2010near,UCBVI}. We now introduce their analogue in the masked setting, which we call \emph{version-space optimism} (VSO). Since the mask is observed noiselessly, VSO is more closely related to the work of \citet{munos2011optimistic} %
than to these standard RL algorithms. It maintains all mask models still compatible with the observations and acts according to the most optimistic one, thereby exploiting only constraints that have already been ruled out by the data. Note that this algorithm is computationally feasible, since the space of masks to search is at most $d$-dimensional and the MDPs are linear. It is also possible to implement it using Extended Value Iteration \citep{jaksch2010near}, see Appendix \ref{app:EVI}.

\begin{algorithm}[ht]
\caption{Version-space optimism (VSO)}
\label{alg:vso}
\begin{algorithmic}[1]
\STATE \textbf{Input:} initial version space $\versionstate_1:=\confset_0$, initial law $\startlaw$.
\FOR{$t=1,2,\dots$}
    \STATE Set $\pi_t\in\argmax_{\pi\in\markovpolicies}\max_{m\in\versionstate_t}J^{ m}_\pi(1,\startlaw)$.
    \STATE Play $\pi_t$ for the episode, collect $(Z_h)_{h\in[H]}:=(S^{\pi_t}_h,\pi_{t,h}(S^{\pi_t}_h))_{h=1}^H$. 
    \STATE $\observedpairs_t\gets \observedpairs_{t-1}\cup \{(Z_h)_{h=1}^H\}$, and $\versionstate_{t+1}\gets\versionstate(\observedpairs_t)$.
\ENDFOR
\end{algorithmic}
\end{algorithm}

\begin{proposition}\label{prop: VSO regret}
    Let $\confset_0=\Bc_2(0,1)$ be the Euclidean ball. For any $(H,T,\phi)$, the version-space optimism algorithm (\cref{alg:vso}) satisfies
    \begin{equation*}
        \sup_{\maskpar^\star\in\confset_0}\ \E_{\maskpar^\star}\bigl[\cumulativeregret(T)\bigr]\ \le\ \min\left\{T,\undecidedrank{\confset_0}\right\}\horizon.
    \end{equation*}
\end{proposition}

\begin{proof}
    By optimism over the version space, either episodes are free ($\pi_t$ coincides with the optimal policy of the true mask) or they land on an element of $\undecided(\versionstate_t)$, which resolves a direction of $\undecided(\confset_0)$. Since there are at most $\undecidedrank{\confset_0}$ directions to resolve, the claim follows.
\end{proof}

\Cref{thm: lower bound for mask learning,prop: VSO regret} together show that the minimax regret rate for the problem is no different from the rate one would get if one could query each state-action pair directly. In this sense, it is of very limited interest, and the heart of the problem lies in understanding how to leverage the dynamics of the MDP to query the version space at minimal regret. 

\subsection{Separation theorem}\label{subsec: separation}

\Cref{thm:separation} will now argue that mask learning can be separated from the problem of learning the dynamics and rewards of the underlying linear MDP, which motivates our choice to study the mask learning problem in isolation. Suppose we have an optimistic learning algorithm $\As$ for linear MDPs that is capable of achieving, over $T$ episodes, a high-probability regret bound of $B_{\mathrm{lin}}(T,H,\delta)$ on an MDP with horizon $H$. We denote the trajectory of $\As$ in episode $t$ by $(S_h^{\As,t},A_h^{\As,t})_{h=1}^\horizon$ and denote by $(B_h^{\As,t})_{h=1}^\horizon$ corresponding step regret upper bounds \citep[see, \eg{},][]{UCBVI}. 

\begin{restatable}{theorem}{separation}\label{thm:separation}
    Consider a \emph{reasonable} algorithm that is jointly optimistic in the mask and the linear MDP. Let $p_t$ be the probability, conditional on the information available before episode $t$, that an action is attempted that is invalid under the true mask.
    If the analysis of $\As$ holds under the joint interaction\footnote{\textit{E.g.}, discarding this data and adapting regret time.}, with 
        \[ \Pb\!\left(\sum_{t=1}^T\beta_t\le B_{\mathrm{lin}}(T,H,\delta/2)\right)\ge 1-\delta/2. \]
     Then
    \[
       \begin{aligned}\regret(T)\le{}& B_{\mathrm{lin}}(T,H,\delta/2)\\
       &+\horizon\left(\min\{T,r_0\}+\sqrt{\frac{T}{2}\log\frac{2}{\delta}}\right)\end{aligned}
    \]
    with probability at least $1-\delta$, wherein $r_0:=\dim\featurematrix(\undecided(\confset_0))$. A high-probability bonus bound for the learner on an unmasked MDP alone does not imply the joint-interaction assumption above.
\end{restatable}

We defer the proof to Appendix~\ref{app:proof-separation}. Let us sketch it briefly. On episodes with no exploration we use the existing bounds of $\As$, trajectorially. On exploratory episodes, we use the usual regret decomposition
    \[ \Delta_t\le \beta_t+\horizon p_t,\qquad
       \beta_t:=\E\!\left[\sum_{h<\tau_t\vee \horizon}B_h^{\As,t}\,\middle|\,\mathcal F_t\right], \]
    where $\Delta_t$ is episode $t$'s conditional pseudo-regret and $\tau_t$ is its first invalid-action time.

    We bound the resulting accumulation of exploration using a martingale argument between samples and $p_t$, then show decrease of the dimensionality potential
\[
 \undecidedpotential{\confset_0}(\versionstate)
 :=
 \dim\featurematrix\bigl(\undecided(\confset_0)\bigr)
 -
 \dim\!\left(
   \featurematrix\bigl(\undecided(\confset_0)\bigr)
   \cap\decideddirections(\versionstate)
 \right),
\]
where $\decideddirections(\versionstate):=\{x\in\Rb^{d}: \forall (\maskpar_1, \maskpar_2) \in \versionstate^{2} , \maskpar_1^\top x= \maskpar_2^\top x\}$ is the subspace whose \emph{value} is known on $\versionstate$.

As a consequence of \cref{thm: lower bound for mask learning,thm:separation}, it would appear that the problem is trivially solved. Instead, it is more accurate to observe that minimax analysis is not the correct lens to analyse the problem, which, indeed, is not statistical. Rather, the core of the problem, and the next section, lies in how to adapt exploration to the interaction of the dynamics with the feature-space of each instance.

\section{Planned exploration}\label{sec: instance dep}

To understand the shortcomings of VSO on non-pathological instances, let us introduce one which is intuitively cheap to explore, but where optimistic algorithms fail. This instance (see \cref{fig: door + chain}) is an extension of \cref{fig:doors}. Of the $\abs{\actionspace}$ actions, only one, say $a_\star$ leads to $s_2$, as before. But now, instead of all other actions being invalid (masked), some proportion $\delta\in(0,1)$ of them leads (with $0$ reward) to a chain\footnote{\Cref{fig: door + chain} shows only one such action, for  legibility.} $(s'_i)_{i=2}^\horizon$ of states which  transition to each other with reward $1$. In the final state of the chain, $s'_H$, $a_\star$ yields reward $1$, and all other actions $0$, and we take any feature map satisfying 
\[\phi(s_1,a_i)=\phi(s_\horizon',v_i),\]
for all $i\in[d]$, $d=\Oc(\abs{\actionspace})$.
Thus, the rewards obtainable are: $0$, if the first action is invalid, $H-2$ if one takes $a\neq a_\star$ at the end of the chain, $H-1$ if one takes $a_\star$ instead, and finally $H$ if one takes $a_\star$ from the start.

On this instance, any algorithm must still begin by sampling doors at random. However, VSO believes that all untried actions lead to $s_2$ directly, so they have value $\horizon$ while the chain always has value $\horizon-1$ at most. Thus, VSO never takes the chain, incurring a regret of order $\Oc(\horizon (d-1)(1-\delta)/2)$ as before. Since a chain can be found (on average) in $\Oc((1-\delta)H)$ tries, and then any further exploration costs at most $\Oc(\delta K)$, hence it is possible to move from a multiplicative dependence in $(\horizon,d)$ to an additive one. 

\begin{figure}[ht]
    \centering
    \includegraphics[width=.5\linewidth]{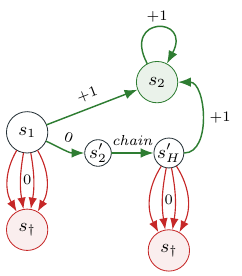}
    \caption{The second example MDP.}\label{fig: door + chain}
\end{figure}

This type of exploration problem is fundamental
 in control problems, and has been documented in the literature of MDPs \citep[see, for instance, the forthcoming work of][]{boone2026towards}.
 In order to address it, we introduce a planning-based algorithm which can be readily interpreted as an analogue of the information-directed sampling (IDS) algorithm of \citeauthor{russo2014information} (\citeyear{russo2014information}, see also \citeauthor{hao2022idrl}, \citeyear{hao2022idrl}).
 Let us first introduce the algorithm, and then we will compare it to both VSO and IDS.

\newcommand{\targetpolicy}{\pi}
\newcommand{\targettime}{\tau_{\targetpolicy}}
In order to determine which are the best exploratory paths through the masked MDP, let us start by specifying a generic version space $\versionstate$, and consider a set of target state-action pairs $\undecided(\versionstate)\subseteq\target$. Given any (non-stationary, randomised) policy $\targetpolicy$, let $\targettime$ be the first time that the trajectory $(S^{\targetpolicy}_h, \targetpolicy_h(S^{\targetpolicy}_h))_{h=1}^\horizon$ generated by $\targetpolicy$ reaches a state-action pair in $\target$. 

Consider the family of auxiliary MDPs $(\MDP_\lambda)_{\lambda\ge0}$ with the same dynamics and rewards as the original MDP, except that actions are limited to $\target\cup\certvalid(\versionstate)$ except that episodes terminate early if $\target$ is reached, yielding a fixed reward of $\lambda$. The value functional of this MDP is readily seen to be
\begin{multline}\label{eq: definition of lambda MDP value functional}
J^{\lambda,\target}_{\versionstate,\cdot}:(\targetpolicy, t, \state)\mapsto \E\bigg[\sum_{h=t}^{(\targettime - 1)\wedge \horizon}\reward(S^{\targetpolicy}_h, \targetpolicy_h(S^{\targetpolicy}_h))\\ + \lambda \1\{\targettime\le \horizon\}\mid S^{\targetpolicy}_t=\state\bigg].
\end{multline}

Via this auxiliary MDP, it is straightforward to obtain a dynamic programming representation of $J^{\lambda,\target}_{\versionstate,\cdot}$, or, in fact, of the associated value function $V^{\lambda,\target}_{\versionstate}(h,\state):= \sup_{\targetpolicy}J^{\lambda,\target}_{\versionstate,\targetpolicy}(h, \state)$, which is given by \cref{prop: dynamic programming for lambda MDP}.

\begin{restatable}{proposition}{lambdadynprog}\label{prop: dynamic programming for lambda MDP}
    The optimal value functions $V^{\lambda,\target}_{\versionstate}$ satisfy the following family, indexed by $\lambda\ge0$ and $\undecided(\versionstate)\subseteq\target$, of dynamic programming equations
    \begin{align}
        V^{\lambda,\target}_{\versionstate}(h, \state) &= \lambda\1_{\{\exists a:\, (\state, a) \in \target\}}\vee \max_{a \in \certvalid(\versionstate_t)} \bigg\{ \reward(\state, a)\notag \\&\hspace{1em}\left.+ \int V^{\lambda,\target}_{\versionstate}(h+1, s') P(\de s'|\state, a)\right\}\hspace{-1em}
        \label{eq: Dynamic Programming for Lambda MDP}
    \end{align}
     with terminal conditions $V^{\lambda,\target}_{\versionstate}(\horizon+1,\cdot)\equiv0$. Moreover, a pointwise maximiser of the right-hand side of \eqref{eq: Dynamic Programming for Lambda MDP} (understood as taking $a\in \target$ if the $\lambda$-term is dominant) is an optimal policy for the auxiliary MDP $\MDP_\lambda$.
\end{restatable}

The dynamic programming (Bellman) equation \eqref{eq: Dynamic Programming for Lambda MDP} might appear highly irregular, and it is indeed uncommon in reinforcement learning. It is the standard form for a mixed control and stopping (a.k.a.\ free-boundary) problem, a well-studied class of stochastic control problems in the mathematical finance literature see, \eg{}, \citet{pham2009continuous}. Compare also pure stopping problems with a dynamic programming equation, such as prophet problems \citep{hill1992survey}.

In any case, for finite-dimensional MDPs, resolution of this equation is numerically straightforward, and we can thus use this solution to compute an optimal policy which explores through the shaped reward of \eqref{eq: definition of lambda MDP value functional}, which targets at each step $\Ts_t:=\undecided(\versionstate_t)$.\footnote{In the appendix, we generalise to bigger sets.}

\begin{algorithm}[ht]
\caption{\texorpdfstring{$\lambda$}{lambda}-planning}
\label{alg:lambdaplanning}
\begin{algorithmic}[1]
\STATE \textbf{Input:} initial version space $\versionstate_1:=\confset_0$, initial law $\startlaw$.
\FOR{$t=1,2,\dots$}
    \STATE $\realprice_t = \realprice_{\versionstate_t}(\undecided(\versionstate_t))$
    \STATE Compute $\targetpolicy_{t}\in\argmax_{\targetpolicy}J^{\realprice_t,\undecided(\versionstate_t)}_{\versionstate_t,\targetpolicy}(\startlaw)$.
    \STATE Play $\pi_t$ for the episode, collect $(Z_h)_{h\in[H]}:=(S^{\pi_t}_h,\pi_{t,h}(S^{\pi_t}_h))_{h=1}^H$. 
    \STATE Update $\observedpairs_t\gets \observedpairs_{t-1}\cup \{(Z_h)_{h=1}^H\}$, and $\versionstate_{t+1}\gets\versionstate(\observedpairs_t)$.
\ENDFOR
\end{algorithmic}
\end{algorithm}

Regret analysis of \cref{alg:lambdaplanning}, see  \cref{prop: lambda planning regret}, bounds the regret of $\lambda$-planning in terms of the \emph{clearing price} for a version space $\versionstate$ and a target $\target$, that is
\[
\realprice_\versionstate(\target) := \min\left\{\lambda \ge 0 : V^{\lambda,\target}_{\versionstate}(1,\startlaw) \ge \max_{\mask\in \versionstate} V^{\mask}(1,\startlaw)\right\} 
\]
(recall that $\startlaw$ is the initial state distribution).

\begin{restatable}{proposition}{lambdaplanningregret}\label{prop: lambda planning regret}
For any $u\in\confset_0$, and any $T\ge1$, the $\lambda$-planning algorithm (\cref{alg:lambdaplanning})  satisfies
\[
\E_{\maskpar}\bigl[\cumulativeregret(T)\bigr]\ \le \E_{\maskpar}\left[\sum_{t=1}^T \realprice_t(\potential_t-\potential_{t+1})\right]
\]
wherein $\potential_t:= \undecidedpotential{\confset_0}(\versionstate_t)
$.
\end{restatable}

We defer the proof to \cref{app:price} for brevity. The high-level argument mirrors the proof of \cref{prop: VSO regret}, and no regret except that we can keep track of the regret of exploration episodes through $\realprice$ rather than upper-bounding by $H$ \eg{} when $\realprice=0$ we have $V^{0,\target}_{\versionstate}(1,\startlaw)
=
\sup_{\mask\in\versionstate} V^{\mask}(1,\startlaw)$ despite exploration.

Now, let us turn to the interpretation of \cref{alg:lambdaplanning} as a form of Information-Directed Sampling (IDS). Let us recall the initial set-up of IDS in stochastic bandits and reinforcement learning. Given a filtration $\episodicfiltration_t$ of information available before episode $t$, and an unknown parameter $\boldsymbol{\maskpar}$ with a given prior, IDS rests on the computation of the mutual information between the history of observations in episodes up to $t$, $\Hc_t$, and a target unknown variable $\chi$ (for us, $\maskpar^\star$). Accounting explicitly for the policy $\pi$ that generates the observations, we define
\[ \Ic_t^\pi(\chi,\Hc){:=}\mathrm{D}_{\mathrm{KL}}\left(\Pb_t^\pi((\chi,\Hc){\in} \cdot)\Vert \Pb_t^\pi(\chi{\in}\cdot)\otimes \Pb_t^\pi(\Hc{\in}\cdot)\right).\]
Here, $\Pb_t^\pi$ denotes the probability measure induced by policy $\pi$ and initial distribution $\startlaw$, conditional on the observed history $\Hc_{t-1}$. The random variable $\Hc$ is any history-valued variable whose restriction up to episode $t-1$ coincides with the observed history. IDS plays the policy 
\[\pi_t \in \argmin_{\pi}\frac{\Eb[( V^{\maskpar^*}(1,\startlaw) - J_\pi^{\maskpar^\star}(1,\startlaw))^2 \vert \Hc_{t-1}]}{\Ic_t^\pi(\maskpar^\star,\Hc)}\]
which minimises a (quasi-)regret-to-information ratio. 

In MDPs, \citet{hao2022idrl} introduced an alternate formulation, where the policy is chosen according to the penalised objective
\begin{align}
    \pi_t \in \argmax_{\pi}\Eb[J_\pi^{\maskpar^*}(1,\startlaw) \vert \Hc_{t-1}] + \lambda \Ic_t^\pi(\maskpar^\star,\Hc),\label{eq: IDS regularised Bao}
\end{align}
which can be recast as an auxiliary MDP with a suitably shaped reward; see Proposition 4.8 therein.

In this formulation, the structural correspondence between \eqref{eq: IDS regularised Bao} and \eqref{eq: definition of lambda MDP value functional} is immediately apparent. In our case, due to the absence of noise in the dynamics (uncertainty is only in the mask), information theory tools and Bayesian methods wither away and leave behind the likelihood of hitting a masked parameter, for which the best proxy is the entry probability into $\target$.
With this interpretation in mind, we can rewrite $\realprice_\versionstate(\target)$ in terms of a ratio, using \cref{lemma: clearing price as ratio}. 

\begin{restatable}{lemma}{clearingpriceratio}
\label{lemma: clearing price as ratio}
For any $\versionstate\subseteq\confset_0$ and $\undecided(\versionstate) \subseteq \target$ such as $V^{0,\target}_{\versionstate}(1,\startlaw) < \sup_{\mask\in\versionstate} V^{\mask}(1,\startlaw)$ the clearing price satisfies 
\begin{align}\label{eq: clearing price as ratio}
\realprice_\versionstate(\target) \!=\! \inf_{\pi} \frac{\sup\limits_{\mask\in\versionstate} V^{\mask}(1,\startlaw) - \Eb\left[\sum\limits_{h=1}^{\targettime - 1} r(S^\pi_h,\pi_h(S^\pi_h))\right]}{\Pb_{\versionstate}(\tau_\pi <H)}.
\end{align}
\end{restatable}

Let us briefly show that the ratio introduced in \cref{lemma: clearing price as ratio} serves as a computable proxy for the ``regret-to-information" ratio underlying IDS.
The numerator in \eqref{eq: clearing price as ratio} is an upper bound on the regret of $\pi$, since 
\[ V^{\maskpar^\star}(1,\startlaw) \le {\underset{\mask\in\versionstate}{\sup}} V^{\mask}(1,\startlaw)\]
and
\[\Eb\left[\sum_{h=1}^{\tau_\pi} r(S^\pi_h,\pi_h(S^\pi_h))\right] \leq \Eb\left[\sum_{h=1}^{H} r(S^\pi_h,\pi_h(S^\pi_h))\right], \]
which can be computed under $\versionstate_t$. Similarly, the denominator is a computable lower bound on the expected number of resolved dimensions, since $\tau_\pi\le H$ and, a.s., $\1\{\tau_\pi\le H\}\le \undecidedpotential{\confset_0}(\versionstate_t) - \undecidedpotential{\confset_0}(\versionstate_{t+1})$. This quantity is used at analysis time in \cref{prop: lambda planning regret}.

\section{Application to Transfer Learning}\label{sec:transfert}
 
In RL, beyond single tasks, the literature has increasingly aimed to improve computational efficiency by learning multiple instances at once or sequentially. The former has generally been studied in \emph{zero-shot} RL \cite{touati2022does}, while the latter falls under \emph{transfer learning}. Given a learnt source and unlearnt target task, transfer aims to leverage shared structures to improve learning, see  \citet{pan2009survey} for a broad survey. In RL, the common method is to pre-train representations of environments on large datasets, and then fine-tune them on small-data target tasks \citep{yosinski2014transferable}. There have been a number of proposed representations: learnt policies \citep{fernandez2006probabilistic}, value functions \citep{taylor2007cross}, latent state features \citep{parisotto2016actor}, successor features \citep{barreto2017successor,barreto2018transfer}, samples \citep{lazaric2008transfer}, or models of the environment \citep{sun2021model}. In all cases, the difficulty lies not only in transferring the correct information, but also in how to spread learning cost between source and target. 

Applied to masked MDPs, we seek to transfer information about two masks for the same linear MDP. The medium of transfer we propose is to choose the set $\confset_0$ of VSO and $\lambda$-planning from information obtained about a prior mask in three experimental frameworks. Throughout, the base MDP 
always has $10$ states, $5$ actions (one of which aborts the MDP independently of the mask), $\startlaw=\delta_{s_1}$, and horizon $10$. The transfer procedure is repeated for $1000$ random MDPs and masks, all figures report averaging on these samples. 
At each random trial, independently for each pair $(s,a)$, we draw three coordinates of $\phi(s,a)$ uniformly, and draw  values for them according to a $\texttt{Dirichlet}(1,1,1)$ distribution, setting the rest to $0$. 
The function $\theta:\statespace\to\Rb^d$ is $0$ everywhere except at two uniformly drawn states, with respective weights $w$ and $1-w$, for $w\sim\texttt{Unif}[0,1]$.
Meanwhile, the reward parameter $\theta_r$ is drawn uniformly in $[0,1]^d$ and the source mask is drawn uniformly at random on the unit disk in the hyperplane orthogonal to the vector of ones.\footnote{Masks near-orthogonal to this plane allow/forbid nearly all actions, leading to degeneration. Similarly, we re-sample masks with no valid actions.} We take this disk as $\confset_0$ in what follows. Our three experiments differ in the sampling process of the target mask with respect to the source $\maskpar^{(1)}$ which is sampled uniformly in $\confset_0$. All transfer version spaces will be intersected with $\confset_0$ to bring the parameter back into the hyperplane. Let $\oplus$ denote the Minkowski sum: $A\oplus B=\{a+b:a\in A,\ b\in B\}$.

\subsection{Single-step transfer}

In our first setting, the source mask $\maskpar^{(1)}$ is assumed to be known, and the target mask $\maskpar^{(2)}$ to lie within distance $\rho$ of it. We therefore initialize the target algorithm from the restricted version space
\[
\confset_0^{\mathrm{t}}
:=
\confset_0\cap
\bigl\{\maskpar:\lVert\maskpar-\maskpar^{(1)}\rVert_2\le\rho\bigr\}.
\]
To assess the benefit of this prior information, we compare each algorithm initialized from $\confset_0^{\mathrm{t}}$ with the same algorithm run on the same target task from the full version space $\confset_0$.

\begin{figure}[H]
\centering
\includegraphics[width=\linewidth]{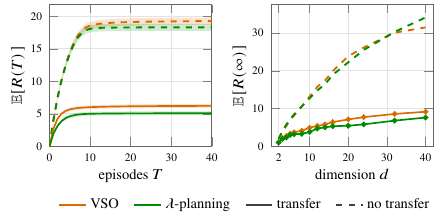}
\caption{Single-step transfer. Left: expected regret against the number of episodes, $d=16$, $\rho=0.13$. Right: total regret against $d\le40$, with $\rho=0.5/\sqrt{d-1}$. Solid: started from $\confset_0^{\mathrm t}$; dashed: no transfer, started from $\confset_0$.}
\label{fig:proposal-single}
\end{figure}
Without transfer (dashed lines), the two algorithms are close, $\lambda$-planning performing slightly better as $d$ grows. With transfer, both perform much better, and $\lambda$-planning is ahead of VSO at every dimension but the smallest, showing the (on average) lower regret of $\lambda$-planning. Moreover, Figure~\ref{fig:proposal-radius}
shows that transfer reduces regret and strengthens the
advantage of $\lambda$-planning when the transferred version space is tight (small $\rho$). This is
consistent with $\lambda$-planning exploiting certified actions to
acquire information at lower regret. As $\rho$ increases, fewer actions
are certified initially and this advantage fades.

\begin{figure}[ht]
\centering
\includegraphics[width=\linewidth]{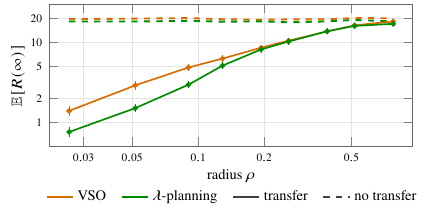}
\caption{Single-step transfer: total regret against the transfer radius $\rho$, $d=16$, log-log axes. Solid: started from $\confset_0^{\mathrm t}$; dashed: no transfer, started from $\confset_0$.}
\label{fig:proposal-radius}
\end{figure}

\subsection{Multiple-step transfer}
Building on the previous setting, we consider a sequence of $K+1$ tasks with mask parameters $\{\maskpar^{(k)}\}_{k=0}^K$, sampled from a uniform source $\maskpar^{(0)}\in\Cc_0$ then sequentially uniformly in the ball $\Bc(u^{(k-1)},\rho)$. We repeat the transfer along the chain by taking the version spaces are as follows: Task $1$ starts from $\confset_0\cap B(\maskpar^{(0)},\rho)$, and task $k+1$ from $(\versionstate_k\oplus B(0,\rho))\cap\confset_0$, where $\versionstate_k$ is the version space transferred from task $k$. This yields a sequential transfer scheme with a set-valued warm start, in the spirit of procedures that leverage information from a previous optimization problem to facilitate the next one \citep{ash2020warm}. Each task is thus solved within a region informed by the preceding task, rather than from the full parameter space.

\begin{figure}[ht]
\centering
\includegraphics[width=\linewidth]{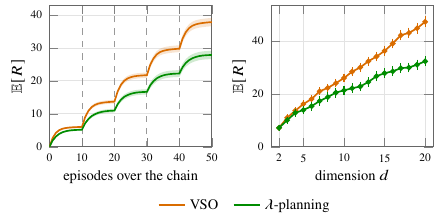}
\caption{Multiple-step transfer, $K=5$ tasks of $10$ episodes. Left: regret accumulated over the chain, $d=16$; dashed vertical lines mark the task switches. Right: total regret over the $K$ tasks against $d$.}
\label{fig:proposal-chain}
\end{figure}
Figure \ref{fig:proposal-chain} shows that  $\lambda$-planning retains its edge over time. Its lead per task stacks across tasks, so that the gap accumulates along the chain (left), and it widens with the dimension (right). The compounding aspect of this procedure shows that $\lambda$-planning can turn a slight advantage in learning in each transfer task into  substantial gains in sequential transfer settings.

We present further experiments on partial-learning of transferable information in Appendix \ref{app:transfer}.

\section{Conclusion}

In this work, we study masked MDPs with unknown action validity and exact linear score feedback. Our separation result gives conditions under which mask-learning costs enter additively when rewards and transitions are also unknown. Version-space optimism achieves minimax-optimal regret, however, minimax optimality can conceal substantial differences in the cost of acquiring information.

To exploit these differences, we introduced $\lambda$-planning, which plans exploration through an auxiliary planning problem. Its clearing price yields an instance dependent regret bound and connects the method to information-directed sampling for deterministic feedback. Transfer experiments illustrate how version spaces preserve useful information across changing masks and how planned exploration can produce gains that accumulate across tasks.

\clearpage

\subsection*{AI use statement}

In this work, we used generative AI tools for several tasks: review or verification of content, and generation of code or proof drafts from formal arguments.
We have not used generative AI tools for theoretical or conceptual development, mathematical claims or proofs (beyond formalisation and edition), while method implementation, translation, dataset cleaning or reformatting, qualitative analysis, and interpretation of results are not applicable to this work.
Additionally, we used generative AI tools for literature discovery, editing text, creating figures and code.
All AI-assisted work was checked line by line by the authors.
We take responsibility for the final content of this work, including text, claims, or artifacts produced with the aid of generative AI.

\subsubsection*{Acknowledgements}
The work of Hélène Halconruy was supported by the French National Research Agency (ANR) under grant ANR-24-CE40-3341 (project DECATTLON) and benefited from Hi! PARIS and State funding managed by the ANR under the France 2030 program (reference ANR-23-IACL-0005).

\bibliographystyle{plainnat}
\bibliography{biblio}

\clearpage

\appendix

\onecolumn
\section*{Appendix}

The appendix is organised as follows. \Cref{app:notation} collects the
notation used throughout the paper. \Cref{sec:proofs} first introduces the
quantities used to track the information gathered on the mask parameter, then
provides the proofs of our results: \cref{app:proof-dp} establishes the dynamic
programming representation for the auxiliary MDP problem
(\cref{prop: dynamic programming for lambda MDP}),
\cref{app:proof-price} derives the clearing-price characterisation
(\cref{lemma: clearing price as ratio}), \cref{app:price}
proves the regret bound of $\lambda$-planning (\cref{prop: lambda planning regret}),
\cref{app:proof-lowerb} proves the lower bound for mask learning
(\cref{thm: lower bound for mask learning}), and \cref{app:proof-separation}
proves the separation result (\cref{thm:separation}). \Cref{app:transfer}
experiments with transfer from a partially learnt source task. Finally,
\cref{app:EVI} presents action optimism, an alternative to version-space
optimism, and links it to extended value iteration.

\section{Notations}\label{app:notation}

Table~\ref{tab:notation} lists the notation of the paper, grouped by topic.
A subscript $t$ denotes the value of a quantity
at episode $t$.

\begingroup
\small
\renewcommand{\arraystretch}{1.0}
\setlength{\LTleft}{0pt}
\setlength{\LTright}{0pt}
\begin{longtable}{@{}>{\raggedright\arraybackslash}p{0.34\textwidth}>{\raggedright\arraybackslash}p{0.62\textwidth}@{}}
\caption{Notations used in the paper.}\label{tab:notation}\\
\toprule
\textbf{Symbol} & \textbf{Meaning} \\
\midrule
\endfirsthead
\multicolumn{2}{@{}l}{\tablename~\thetable\ (continued)} \\
\toprule
\textbf{Symbol} & \textbf{Meaning} \\
\midrule
\endhead
\bottomrule
\endfoot

\multicolumn{2}{@{}l}{\textbf{MDP and masks}} \\
$\MDP=(\statespace,\actionspace,\transition,\reward,\horizon)$ & finite-horizon MDP \\
$h\in[\horizon]$, $[n]$ & step; $\{1,\dots,n\}$ \\
$\startlaw$, $s_1$ & initial law; initial state \\
$\feature(\state,\action)\in\Rb^d$, $d$ & feature map; dimension \\
$\theta_r\in\Rb^d$, $\theta_P:\statespace\to\Rb^d$ & reward and transition parameters \\
$\mask\in2^{\statespace\x\actionspace}$ & mask \\
$\MDP_\mask$, $r_\mask$, $P_\mask$ & masked MDP, reward, transitions \\
$\cemeterystate$ & cemetery state \\
$\widetilde\feature$ & masked feature map \\
$\maskpar^\star$ (also $\theta_\mask$) & true mask parameter \\
$\maskpar$ & mask parameter \\

\addlinespace
\multicolumn{2}{@{}l}{\textbf{Policies, values and regret}} \\
$\policy=(\policy_h)_{h\in[\horizon]}$, $\markovpolicies$ & policy; Markov policies \\
$\State^{\policy}_h$, $Z_h$ & state at step $h$; pair played at step $h$ \\
$J_\policy$, $V$ & value of $\policy$; optimal value \\
$t$, $k$; $T$, $K$ & episode; number of episodes \\
$\regret_K(\Ac)$, $\cumulativeregret(T)$ & cumulative pseudo-regret \\
$\E_{\maskpar}$, $\Pb_{\maskpar}$ & expectation, probability under $\maskpar$ \\
$\episodicfiltration_t$ & information before episode $t$ \\

\addlinespace
\multicolumn{2}{@{}l}{\textbf{Version spaces}} \\
$\confset_0$ & prior set of mask parameters \\
$\Bc_2(0,1)$, $B(c,r)$ & unit ball; ball of centre $c$, radius $r$ \\
$\observedpairs$, $\observedpairs_t$ & observed pairs \\
$\Phi_{\observedpairs}$, $y$ & features of $\observedpairs$; revealed values \\
$\versionstate(\observedpairs)$, $\versionstate_t$ & version space \\
$\pair$, $\pairfeature{\pair}$ & pair; its feature \\
$\certvalid(\versionstate)$, $\certinvalid(\versionstate)$ & certified valid, certified invalid pairs \\
$\undecided(\versionstate)$ & uncertified pairs \\
$\featurematrix(A)$, $\undecidedrank{\versionstate}$ & features of $A$; dimension of their span \\
$\decideddirections(\versionstate)$ & directions known on $\versionstate$ \\
$\undecidedpotential{\confset_0}(\versionstate)$, $\potential_t$ & potential \\
$\informative(\versionstate)$ & informative pairs \\
$\Gamma_t$ & $\dim\feature(\undecided(\versionstate_t))$ \\

\addlinespace
\multicolumn{2}{@{}l}{\textbf{Algorithms}} \\
$\target$ & target set \\
$\targettime$ & first time $\target$ is reached \\
$\MDP_\lambda$ & auxiliary MDP \\
$J^{\lambda,\target}_{\versionstate,\policy}$, $V^{\lambda,\target}_{\versionstate}$ & value of $\policy$; optimal value in $\MDP_\lambda$ \\
$\realprice_\versionstate(\target)$, $\realprice_t$ & clearing price \\
$\As$, $B_{\mathrm{lin}}(T,H,\delta)$, $B^{\As,t}_h$ & linear-MDP learner; its regret bound; step bounds \\
$\Hc_t$, $\Ic^\policy_t(\chi,\Hc)$ & history; mutual information \\
\addlinespace
\multicolumn{2}{@{}l}{\textbf{Transfer experiments}} \\
$\rho$ & transfer radius\\
$\confset_0^{\mathrm t}$ & transfer prior \\
$\maskpar^{(\mathrm k)}$, $K$, $T_s$ & chain of masks; number of tasks; episodes per task \\
$A\oplus B$ & Minkowski sum \\
\end{longtable}
\endgroup

\section{Proofs}
\label{sec:proofs}
Let us first introduce the quantities used to track information
about the mask parameter. The subspace
\[
\decideddirections(\versionstate)
:= \bigl\{\, \theta \in \R^d \;:\; \maskpar \mapsto \theta^\top \maskpar
   \text{ is constant on } \versionstate \,\bigr\},
\]
contains the
directions whose scalar values are completely determined by the current version
space. The potential
\[
 \undecidedpotential{\confset_0}(\versionstate)
 :=
 \dim\featurematrix\bigl(\undecided(\confset_0)\bigr)
 -
 \dim\!\left(
   \featurematrix\bigl(\undecided(\confset_0)\bigr)
   \cap\decideddirections(\versionstate)
 \right).
\]
counts the remaining
unknown directions (when in $\versionstate$) in the span of the initially uncertified (in $\confset_0$) features,
while 
\[ \informative(\versionstate) := \bigl\{ \pair : \pairfeature{\pair} \in \decideddirections(\versionstate) + \featurematrix(\undecided(\versionstate)), \ \pairfeature{\pair} \notin \decideddirections(\versionstate) \bigr\} \supseteq \undecided(\versionstate). \]
identifies exactly the pairs that resolve at least one
such direction. This construction allows us to relate visits to $\informative(\versionstate)$ to decreases
in $\undecidedpotential{\confset_0}(\versionstate)$.

These quantities distinguish knowledge of an action's validity from
knowledge of its underlying scalar value: an already certified action may still provide useful scalar information, such as certifying further actions for example if $\pairfeature{\pair_3}^\top \maskpar^\star> \lVert \pairfeature{\pair_3} - \pairfeature{\pair_4}  \rVert$ in Figure~\ref{fig:trichotomy}. 

Throughout, $\versionstate_t$ is the knowledge state entering episode $t$, $\potential_t:=\undecidedpotential{\confset_0}(\versionstate_t)$, and $\maskpar^\star\in\confset_0$ the truth.

\subsection{Proof of Proposition~\ref{prop: dynamic programming for lambda MDP}}
\label{app:proof-dp}
The auxiliary MDP is the planning tool underlying
$\lambda$-planning. Proposition~\ref{prop: dynamic programming for lambda MDP} provides the dynamic programming
representation needed to compute its optimal policy for a fixed
exploration reward.
\lambdadynprog*
\begin{proofof}{Proposition~\ref{prop: dynamic programming for lambda MDP}}
The following recursion in 
\[
V^{\lambda,\target}_{\versionstate}(h, \state) 
= \lambda\1_{\{\exists a:\, (\state, a) \in \target\}}\vee \max_{a \in \certvalid(\versionstate_t)} \bigg\{ \reward(\state, a)\left.+ \int V^{\lambda,\target}_{\versionstate}(h+1, s') P(\de s'|\state, a)\right\},
\]
is exactly the value-iteration scheme for a finite-horizon MDP with known dynamics, in which each state retains its certified actions, together with their rewards and transitions, with the possibility of absorption into the terminal state with reward $\lambda$ whenever an action of $\target$ is present. 

A policy of this auxiliary MDP has a stopping time $\targettime$ at which it takes its first action absorbing action, with $\targettime=\infty$ if no contact is made. Its expected cumulative reward is the reward collected before $\targettime$, plus $\lambda$ on the event $\tau\le\horizon$, namely 
\[
J^{\lambda,\target}_{\versionstate,\cdot}(\targetpolicy, t, \state) =  \E\bigg[\sum_{h=t}^{(\targettime-1)\wedge \horizon}\reward(S^{\targetpolicy}_h, \targetpolicy_h(S^{\targetpolicy}_h)) + \lambda \1\{\targettime\le \horizon\}\mid S^{\targetpolicy}_t=\state\bigg].
\]

\end{proofof}
\subsection{Proof of Lemma~\ref{lemma: clearing price as ratio}}
\label{app:proof-price}
Here we prove the view of the clearing price that is the basis for the comparison with information-directed sampling.
\clearingpriceratio*
\begin{proofof}{Lemma~\ref{lemma: clearing price as ratio}}
\label{proof:price}
\noindent
Let $\target$ be a set state-action pairs containing $\undecided(\versionstate)$, denote the function $f:\lambda\mapsto V^{\lambda,\target}_{\versionstate}(1,\startlaw)$. For a policy $\pi$, note its stopping time be $\targettime$ at which it takes its first action in $\target$. Noting
\[L = \inf_{\pi} \frac{\sup_{\mask\in\versionstate} V^{\mask}(1,\startlaw) - \Eb\left[\sum_{h=1}^{(\tau_\pi -1)} r(S^\pi_h,\pi_h(S^\pi_h))\right]}{\Pb_{\versionstate}(\tau_\pi \leq H)}
\]
Let us show that
\[
 \{\lambda\ge0: f(\lambda)\ge \sup_{\mask\in\versionstate} V^{\mask}(1,\startlaw)\}
 =
 \left[L,\infty \right),
\]
Taking the minimum over $\lambda$ will close the argument.
By Proposition~\ref{prop: dynamic programming for lambda MDP}, $f$ is a supremum over policys $\policy$ of the affine functions
\[
g_\policy:\lambda\mapsto \lambda \Pb\left(\targettime\le\horizon\right) + \Eb\left[\sum_{h=1}^{\targettime-1} r(S^\pi_h,\pi_h(S^\pi_h))\right],
\]
whose slopes $\Pb\left(\targettime\le\horizon\right)$ are probabilities, hence lie in $[0,1]$.
Thus, for all $\lambda, \lambda' \ge 0$ and every route $\policy$,
\[
  g_\policy(\lambda) \le g_\policy(\lambda') + |\lambda - \lambda'|
  \le f(\lambda') + |\lambda - \lambda'| .
\]
Taking the supremum over routes gives
$f(\lambda) \le f(\lambda') + |\lambda - \lambda'|$, where $f(\lambda') \le H + \lambda' < \infty$.
Exchanging $\lambda$ and $\lambda'$ yields $|f(\lambda) - f(\lambda')| \le |\lambda - \lambda'|$,
so $f$ is $1$-Lipschitz, and in particular continuous.This will be used later.
\\
Suppose $f(\lambda)\ge \sup_{\mask\in\versionstate} V^{\mask}(1,\startlaw)$. By the finite-horizon dynamic-programming recursion, the supremum is attained by some policy $\policy^{\lambda}$ such that 
\[
f(\lambda) = g_{\policy^\lambda}(\lambda).
\]
Thus
\[
\lambda \Pb\left(\tau_{\pi^{\lambda}}\le\horizon\right)
 \ge \sup_{\mask\in\versionstate} V^{\mask}(1,\startlaw)-\Eb\left[\sum_{h=1}^{\tau_{\pi^{\lambda}}-1} r(S_{h}^{\pi^{\lambda}_h},\pi^{\lambda}_h(S_{h}^{\pi^{\lambda}_h}))\right]
 \ge \sup_{\mask\in\versionstate} V^{\mask}(1,\startlaw)-f(0) > 0 
\]
where the last inequality is in the lemmas assumptions.
This shows $\Pb\left(\targettime\le\horizon\right)>0$, thus
\[
\frac{\sup_{\mask\in\versionstate} V^{\mask}(1,\startlaw) -  \Eb\left[\sum_{h=1}^{\tau_{\pi^{\lambda}}-1} r(S_{h}^{\pi^{\lambda}_h},\pi^{\lambda}_h(S_{h}^{\pi^{\lambda}_h}))\right]}{\Pb\left(\tau_{\pi^{\lambda}}\le\horizon\right)} \le \lambda.\]
Taking the infimum over both sides gives us $\lambda\ge L$.
\\
Conversely, let $\lambda>L$. The infimum definition of $L$ implies that there exists some  route $\policy$  such as 
\[
\frac{\sup_{\mask\in\versionstate} V^{\mask}(1,\startlaw) - \Eb\left[\sum_{h=1}^{\tau_\pi-1} r(S^\pi_h,\pi_h(S^\pi_h))\right]}{\Pb_{\versionstate}(\tau_\pi <H)} <\lambda .
\]
Rearranging and taking the supremum gives us $f(\lambda)> \sup_{\mask\in\versionstate} V^{\mask}(1,\startlaw)$. By continuity of $f$, we also have $f(L)\ge \sup_{\mask\in\versionstate} V^{\mask}(1,\startlaw)$ which gives us the same large inequality the edge case of $\lambda = L$.
\end{proofof}

\subsection{Proof of Proposition~\ref{prop: lambda planning regret}}\label{app:price}
We prove the main performance guarantee for
$\lambda$-planning.
\lambdaplanningregret*
\begin{proofof}{Proposition~\ref{prop: lambda planning regret}}
Fix $\maskpar\in\confset_0$, and let
$\episodicfiltration_t$ denote the information available before episode
$t$, including the choice of $\pi_t$. Set
\[
 \realprice_t
 =\realprice_{\versionstate_t}\bigl(\informative(\versionstate_t)\bigr)
 \quad\text{and}\quad
 \potential_t=\undecidedpotential{\confset_0}(\versionstate_t),
\]
where
\[
 \undecidedpotential{\confset_0}(\versionstate)
 :=
 \dim\featurematrix\bigl(\undecided(\confset_0)\bigr)
 -
 \dim\!\left(
   \featurematrix\bigl(\undecided(\confset_0)\bigr)
   \cap\decideddirections(\versionstate)
 \right).
\]
and 

Let $\tau_{\pi_t}$ be the first time that $\pi_t$ reaches
$\informative(\versionstate_t)$, with $\tau_{\pi_t}=\infty$ if no such
pair is reached during the episode. The reach probability is

$\Pb_{\maskpar}(\tau_{\pi_t}\le\horizon\mid\episodicfiltration_t)$. The fact that we condition on $\episodicfiltration_t$ is because we are no longer considering any $\undecided(\versionstate) \subseteq \target$
but the random variable $\informative(\versionstate_t)$.

By definition of $\pi_t$ we have 
\[
J^{\realprice_t,\target}_{\versionstate,\pi_t}(h, \state) \geq \sup_{\mask\in\versionstate} V^{\mask}(1,\startlaw) \geq V^{\maskpar}(1,\startlaw)
\]
In the proof of Lemma~\ref{lemma: clearing price as ratio} we show 
that $\pi_t$ is such that 
$\realprice_{\versionstate_t}\bigl(\informative(\versionstate_t)\bigr)$:

 \begin{align*}
 J^{\realprice_t,\target}_{\versionstate,\pi_t}(h, \state) & = \realprice_t \Pb_{\maskpar}(\tau_{\pi_t}\le\horizon\mid\episodicfiltration_t) + \Eb_{\maskpar}\left[\sum_{h=1}^{\tau_\pi-1} r(S^\pi_h,\pi_h(S^\pi_h))\mid\episodicfiltration_t\right] \\
 & \leq \realprice_t \Pb_{\maskpar}(\tau_{\pi_t}\le\horizon\mid\episodicfiltration_t) + J_{\pi_t}^{\maskpar}(1,\startlaw)
 \end{align*}
thus 
\begin{equation}
 \label{eq:lambda-proof-one-episode}
 \begin{aligned}
 &V^{\maskpar}(1,\startlaw)
       -J_{\pi_t}^{\maskpar}(1,\startlaw)\\
 &\qquad\le\realprice_t\,
 \Pb_{\maskpar}\!\left(
       \tau_{\pi_t}\le\horizon\mid\episodicfiltration_t\right).
 \end{aligned}
\end{equation}
As in the previous proof, consider the case where we are on a non zero-regret episodes, so its
right-hand side is nonnegative.

It remains to relate the probability in
\eqref{eq:lambda-proof-one-episode} to the increment in $\potential_t$ and taking the expectation will yield the result. 
Since $\versionstate_{t+1}\subseteq\versionstate_t$,
\[
 \decideddirections(\versionstate_t)
 \subseteq\decideddirections(\versionstate_{t+1}),
\]
so $\potential_{t+1}\le\potential_t$ on every sample path.

Recall 
\[ \informative(\versionstate) := \bigl\{ \pair : \pairfeature{\pair} \in \decideddirections(\versionstate) + \mathrm{span}\,\featurematrix(\undecided(\versionstate)), \ \pairfeature{\pair} \notin \decideddirections(\versionstate) \bigr\} \supseteq \undecided(\versionstate). \]
On the event $\{\tau_{\pi_t}\le\horizon\}$, let $\pair$ denote the
first pair reached in $\informative(\versionstate_t)$. Thus we can write
\[
 \pairfeature{\pair}=w+v,
 \qquad w\in\decideddirections(\versionstate_t),
\]
\[
 v\in\mathrm{span}\,\featurematrix(\undecided(\versionstate_t)),
 \qquad v\notin\decideddirections(\versionstate_t).
\]
Because $\versionstate_t \subseteq \confset_0$ we have
$\undecided(\versionstate_t)\subseteq\undecided(\confset_0)$,
so $v\in\mathrm{span}\,\featurematrix(\undecided(\confset_0))$.
Playing $\pair$ reveals $\pairfeature{\pair}^\top \maskpar$,
, hence
$\pairfeature{\pair}\in\decideddirections(\versionstate_{t+1})$.
Since $w$ also belongs to this subspace,
\[
 v=\pairfeature{\pair}-w
 \in\mathrm{span}\,\featurematrix(\undecided(\confset_0))
       \cap\decideddirections(\versionstate_{t+1}),
\]
whereas
$v\notin\mathrm{span}\,\featurematrix(\undecided(\confset_0))
       \cap\decideddirections(\versionstate_t)$.
Recall
\[
 \undecidedpotential{\confset_0}(\versionstate)
 :=
 \dim\featurematrix\bigl(\undecided(\confset_0)\bigr)
 -
 \dim\!\left(
   \featurematrix\bigl(\undecided(\confset_0)\bigr)
   \cap\decideddirections(\versionstate)
 \right).
\]
The intersection in
$\undecidedpotential{\confset_0}$ therefore gains at least one
dimension. This proves
\begin{equation}
 \label{eq:lambda-proof-potential-drop}
 \1\{\tau_{\pi_t}\le\horizon\}
 \le\potential_t-\potential_{t+1}.
\end{equation}

Taking conditional expectations in
\eqref{eq:lambda-proof-potential-drop} and substituting into
\eqref{eq:lambda-proof-one-episode}, we obtain
\[
 \begin{aligned}
 &V^{\maskpar}(1,\startlaw)
       -J_{\pi_t}^{\maskpar}(1,\startlaw)\\
 &\quad\le\E_{\maskpar}\!\left[
       \realprice_t(\potential_t-\potential_{t+1})
       \,\middle|\,\episodicfiltration_t\right],
 \end{aligned}
\]
where we used that $\realprice_t$ is nonnegative and
$\episodicfiltration_t$-measurable. Finally, summing over episodes
and applying the tower property gives
\begin{align*}
 &\E_{\maskpar}[\cumulativeregret(T)]\\
 &=\sum_{t=1}^T\E_{\maskpar}\!\left[
     V^{\maskpar}(1,\startlaw)
       -J_{\pi_t}^{\maskpar}(1,\startlaw)\right]\\
 &\le\E_{\maskpar}\!\left[
       \sum_{t=1}^T\realprice_t
                (\potential_t-\potential_{t+1})\right].
\end{align*}
Hence the result.
\end{proofof}

\begin{remark}
Proposition~\ref{prop: dynamic programming for lambda MDP} and Lemma~\ref{lemma: clearing price as ratio} are proved for arbitrary target sets, the agent
being allowed to stop at any pair of $\undecided(\versionstate_t)\subseteq \target$, the same proof gives the same bound
for any $\mathcal{F}_t$-measurable targets
$\undecided(\versionstate_t)\subseteq \target_{t}\subseteq I(\versionstate_t)$, with
prices $\mathring{\lambda}_{\versionstate_t}(\target_{t})$. These prices are
non-increasing in $\target_{t}$, so $\target_{t}=I(\versionstate_t)$ yields the smallest ones.
\end{remark}

\subsection{Proof of \texorpdfstring{\cref{thm: lower bound for mask learning}}{the lower bound}}
\label{app:proof-lowerb}
Here we provide explicit construction of the MDP offering the lower bound on the regret of any learning algorithm.
\lowerb*
\begin{proofof}{\cref{thm: lower bound for mask learning}}
    \label{proof:lowerb}
    Let $N := d - 1 \ge 2$, take the initial distribution
    $\rho = \delta_{s_1}$, and consider the masked MDP of
    \cref{fig:doors}: $\actionspace = \{\action_1,\dots,\action_N\}$;
    from $s_1$, every action leads deterministically to $s_2$ with
    reward $1$, and in $s_2$ every action loops deterministically on
    $s_2$ with reward $1$.

    Take
    \[
        \feature(s_1,\action_j) = e_j,
        \qquad
        \feature(s_2,\action_j) = e_d,
        \qquad j \in [N].
    \]
    This is a linear MDP: for instance, one may take
    \[
        \theta_r = \sum_{i=1}^d e_i,
        \qquad
        \theta_P(s_2) = \sum_{i=1}^d e_i,
        \qquad
        \theta_P(s) = 0 \quad \text{for } s \neq s_2.
    \]
    Then
    \[
        r(s,\action)
        = \feature(s,\action)^\top \theta_r
        = 1
    \]
    on both $s_1$ and $s_2$, while
    \[
        P(s,\action,s_2)
        = \feature(s,\action)^\top \theta_P(s_2)
        = 1.
    \]
    Moreover, every pair is uncertified under $\confset_0$, and
    \[
        \operatorname{span}
        \bigl(\feature(\undecided(\confset_0))\bigr)
        = \operatorname{span}(e_1,\dots,e_d),
    \]
    so $\undecidedrank{\confset_0} = d$.

    For $J \in [N]$, let
    \[
        \maskpar^{(J)}
        := \frac{1}{\sqrt d}
        \Bigl(
            e_J
            - \sum_{j \in [N]\setminus\{J\}} e_j
            + e_d
        \Bigr)
        \in \confset_0,
    \]
    under which $\action_J$ is the only valid action in $s_1$, while
    every action is valid in $s_2$. Draw $J$ uniformly in $[N]$.
    Every episode in which an invalid action is played in $s_1$
    incurs regret $\horizon$, since the optimal value is $\horizon$
    whereas the episode then moves to the cemetery state and collects
    no reward.

    Playing $\action_j$ with $j \neq J$ reveals the score
    $-1/\sqrt d$, independently of the value of $J$, and rewards and
    transitions are deterministic and known. Conditional on the
    internal randomisation of the algorithm, let
    $(k_1,k_2,\dots)$ denote the sequence of actions selected in
    $s_1$ along the counterfactual history in which every action tried
    so far is invalid. Since the feedback following any such failure
    is always the same, this sequence is independent of $J$.

    If $J \notin \{k_1,\dots,k_t\}$,
    then the realised history coincides with this all-failures history
    through episode $t$, and episode $t$ incurs regret $\horizon$.
    Since $J$ is uniform on $[N]$ and
    $\lvert\{k_1,\dots,k_t\}\rvert \le t$,
    \[
        \mathbb P\bigl(J \notin \{k_1,\dots,k_t\}\bigr)
        \ge \frac{N-t}{N}.
    \]
    Therefore,
    \begin{align*}
        \E\bigl[\cumulativeregret(T)\bigr]
        &\ge \horizon
              \sum_{t=1}^{\min\{T,N\}}
              \frac{N-t}{N} \\
        &= \horizon\,\frac{m(2N-m-1)}{2N} \\
        &\ge \horizon\,\frac{m(N-1)}{2N} \\
        &\ge \frac{\horizon m}{4},
    \end{align*}
    where $m := \min\{T,N\} \le N$ and $N \ge 2$.

    Since the supremum over $\maskpar^\star \in \confset_0$ dominates
    the average over $J$, this gives
    \[
        \sup_{\maskpar^\star \in \confset_0}
        \E_{\maskpar^\star}
        \bigl[\cumulativeregret(T)\bigr]
        \ge \frac{\horizon}{4}\min\{T,d-1\}.
    \]
    The second inequality follows from
    \[
        \min\{T,d-1\}
        \ge \frac12\min\{T,d\},
        \qquad d \ge 2.
    \]
\end{proofof}

\subsection{Proof of \texorpdfstring{\cref{thm:separation}}{the separation theorem}}
\label{app:proof-separation}

Here we prove that that mask learning can be separated from the problem of learning the the underlying MDP, motivating our choice to study the mask learning problem in isolation. 

\separation*
\begin{proofof}{\cref{thm:separation}}
Let
\[
    \Psi_t:=\undecidedpotential{\confset_0}(\versionstate_t).
\]
Let $X_t$ be the indicator of the event that an invalid action is attempted in episode $t$. 
If $X_t=1$, the attempted action is truly invalid, so its exact observed
mask score is strictly negative. Its feature was initially uncertified:
it could not have been initially certified valid, and the algorithm does
not select certified-invalid pairs. Moreover, its feature is not in
$\decideddirections(\versionstate_t)$, since otherwise its score would
already be known to be negative and the pair would be certified invalid.
The feature therefore lies in the span of the initially undecided
features but adds a new known direction when its scalar score is
observed. Thus $\Psi_t-\Psi_{t+1}\ge 1$ whenever $X_t=1$. The potential
is nonincreasing and starts at $r_0$, which gives, pathwise,
\[
    \sum_{t=1}^T X_t\le \min\{T,r_0\}.
\]

The sequence $X_t-p_t$ is a martingale-difference sequence with
increments of range length one. By the one-sided Azuma--Hoeffding
inequality, with probability at least $1-\delta/2$,
\[
    \sum_{t=1}^T p_t
    \le \sum_{t=1}^T X_t
       +\sqrt{\frac{T}{2}\log\frac{2}{\delta}}
    \le \min\{T,r_0\}
       +\sqrt{\frac{T}{2}\log\frac{2}{\delta}}.
\]
Intersect this event with the assumed bonus event, also of probability
at least $1-\delta/2$. On the intersection, sum the episode-wise
decomposition $\Delta_t\le\beta_t+\horizon p_t$ over $t$ to obtain
the claimed bound. A union bound gives probability at least $1-\delta$.
\end{proofof}

\section{Partially learnt source}\label{app:transfer}
In this Section we consider a single-transfer setting in which the source task $\maskpar^{(1)}$ is only partially learnt before transfer. Starting from $\confset_0$, the source algorithm is stopped once $\partialdirections$ of the $d-1$ mask directions have been identified. The resulting version space $\versionstate_{0}$ 
is then transferred to the target, drawn uniformly in $\confset_0\cap B(\maskpar^{(1)},\rho)$ whose version space becomes
$(\versionstate_{0}\oplus B(0,\rho))\cap \confset_0$.
This partial transfer preserves the information already learnt on the source even though the entire source mask is not learnt, making it less restrictive than transferring the source solution entirely. As in the single-step setting, we compare this initialization with the same algorithm run on the same target from the full version space $\confset_0$. 

\begin{figure}[H]
\centering
\includegraphics[width=.75\linewidth]{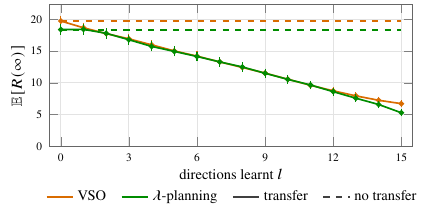}
\caption{Partially learnt source: total regret against the number $\partialdirections$ of source directions learnt before transfer, $d=16$. Solid: started from $(\versionstate_0\oplus B(0,\rho))\cap\confset_0$; dashed: no transfer, started from $\confset_0$.}
\label{fig:proposal-partial}
\end{figure}

Figure~\ref{fig:proposal-partial}
shows that useful transfer does not require a fully learnt
source: additional identified directions progressively reduce target
regret for both algorithms. Their similar performance for most values
of $l$ indicates that much of the improvement comes from the transferred
information itself. A slight advantage for $\lambda$-planning appears
when all source directions have been identified, presumably arising from the large a proportion of certified actions that it can use to plan.

\section{Action optimism - Extended Value Iteration}\label{app:EVI}

Version-space optimism requires finding a single consistent mask under
which the optimal value is maximal. Although planning for a fixed mask
is straightforward, searching over compatible masks can be costly. This motivates a relaxation that applies optimism
locally, in the spirit of extended value iteration: retain an action
whenever some parameter in the current version space makes it valid.
Once these action sets are constructed, the policy is obtained by a
single backward value-iteration pass. This relaxation avoids a separate
planning problem for each compatible mask.

The policy  played is the optimal policy $\policy_\versionstate = VI(\MDP(\optimisticgate_\versionstate))$ of the MDP with the following optimistic action set:
\[
\optimisticgate_\versionstate(\state):=\bigcup_{\maskpar\in\versionstate}\validactions{\maskpar}(\state)
=\{\action:\ \exists\,\maskpar\in\versionstate,\ \inner{\feature(\state,\action)}{\maskpar}\ge0\}
\]
Note that there is no reason why one mask structurally admits them all, as different actions may be
justified by different parameters, so the algorithm may plan with a model that does not exist in the version space. 
The relaxation is also strictly looser, noting $\optimisticvalue_{1,\versionstate}(\startlaw) = V^{\policy_\versionstate, \MDP(\optimisticgate_\versionstate)}(\startlaw)$ we have
\[
\sup_{\mask\in\versionstate}V^{\mask}(1,\startlaw)\le\optimisticvalue_{1,\versionstate}(\startlaw)
\]

\end{document}